\documentclass{article}
    \PassOptionsToPackage{numbers, compress}{natbib}

\usepackage[main, final]{tagds_2025_corrected}

 \workshoptitle{Topology, Algebra, and Geometry in Data Science}

\usepackage[utf8]{inputenc} 
\usepackage[T1]{fontenc}    
\usepackage{hyperref}       
\usepackage{url}            
\usepackage{booktabs}       
\usepackage{amsfonts}       
\usepackage{nicefrac}       
\usepackage{microtype}      
\usepackage{xcolor}         
\usepackage{amsmath,amssymb,amsthm}
\usepackage{graphicx}
\usepackage{subcaption}

\newtheorem{proposition}{Proposition}
\newtheorem{corollary}{Corollary}
\newtheorem{assumption}{Assumption}
\newtheorem{definition}{Definition}
\newtheorem{lemma}{Lemma}
\newtheorem{remark}{Remark}

\newtheorem*{informaltheorem}{Proposition}

\title{Multi-View Graph Learning with Graph-Tuple}

\author{%
Shiyu Chen \\
Department of Applied Mathematics and Statistics \\
Johns Hopkins University \\
\texttt{schen355@jh.edu} 
\And
  Ningyuan (Teresa) Huang \\
  Flatiron Institute \\
  \texttt{thuang@flatironinstitute.org} \\
   \AND
   Soledad Villar \\
  Department of Applied Mathematics and Statistics \\
  Johns Hopkins University \\
  and Flatiron Institute\\
  \texttt{soledad.villar@jhu.edu}\\
}

\newcommand{\blind}{0}

\begin{document}

\maketitle

\begin{abstract}
Graph Neural Networks (GNNs) typically scale with the number of graph edges, making them well suited for sparse graphs but less efficient on dense graphs, such as point clouds or molecular interactions. A common remedy is to sparsify the graph via similarity thresholding or distance pruning, but this forces an arbitrary choice of a single interaction scale and discards crucial information from other scales. To overcome this limitation, we introduce a multi-view graph-tuple framework. Instead of a single graph, our graph-tuple framework partitions the graph into disjoint subgraphs, capturing primary local interactions and weaker, long-range connections. We then learn multi-view representations from the graph-tuple via a heterogeneous message-passing architecture inspired by the theory of non-commuting operators, which we formally prove is strictly more expressive and guarantees a lower oracle risk compared to single-graph message-passing models. We instantiate our framework on two scientific domains: molecular property prediction from feature-scarce Coulomb matrices and cosmological parameter inference from geometric point clouds. On both applications, our multi-view graph-tuple models demonstrate better performance than single-graph baselines, highlighting the power and versatility of our multi-view approach.

\end{abstract}

\section{Introduction}
Graph neural networks (GNNs) have demonstrated remarkable success in learning from structured data~\cite{bronstein2017geometric}, achieving state-of-the-art results across diverse fields such as social network analysis, recommendation systems, drug discovery, and materials science~\citep{kipf2016semi, hamilton2017inductive, he2020lightgcn, gilmer2017neural, xie2018crystal}. The power of GNNs stems from their ability to learn rich representations of nodes and entire graphs by iteratively passing and aggregating messages over a relational structure~\citep{gilmer2017neural, xu2018powerful}. This core mechanism endows them with a strong relational inductive bias~\citep{battaglia2018relational}: the inherent assumption that an object’s properties are determined by its connections and local context. This bias is precisely why GNNs are so effective for tasks on graph-structured data~\citep{kipf2016semi}.

Computationally, GNNs typically scale with the number of graph edges, making them efficient for sparse graphs. However, this efficiency degrades as the graphs become denser. This poses a particular challenge on large, dense graphs such as fully-connected distance graphs derived from point clouds. 
To make GNN training efficient on dense graphs, a common way is to sparsify it by applying similarity thresholding or distance pruning  \cite{huang2025cosmobench}. Yet, this often reduces information and results in graph representations based on a single fixed scale of interaction. For example, a high threshold applied to a molecule's Coulomb matrix retains only strong chemical bonds at the expense of losing important weaker connections. Conversely, a low threshold preserves these weaker connections, but also introduces significant noise. Alternatively, invariant feature models~\cite{blum2024learning} reduce the computational cost to linear in the number of points, by exploiting the low-rank structure of the point cloud and allowing exact reconstruction of the full adjacency matrix from a small submatrix and anchor points.

A complementary line of works avoids graph sparsification and low-rank assumptions, by decomposing a single dense graph into multiple sparser graphs and then learning them in parallel, such as multi-view methods ~\cite{yun2019graph, hassani2020contrastive, park2020unsupervised} and heterogeneous GNNs~\cite{schlichtkrull2018modeling, wang2019heterogeneous}. These approaches preserve diverse interaction ranges while being computationally tractable, yet are typically designed for heterogeneous graphs with multiple node types or edges types, not directly applicable on homogeneous graphs with continuous edge features. 

To tackle these challenges, we propose a multi-view graph representation that captures both fine-grained and contextual interactions. Instead of a single graph, we construct a graph tuple over the same nodes by explicitly partitioning edges according to interaction strength (e.g., distance or Coulomb energy): a strong-connection graph retaining the strongest local interactions and a complementary weak-connection graph providing broader context. Inspired by the theoretical insights of the GtNN framework~\cite{velasco2024graph}, we then explicitly integrate multiple message-passing operations in a single layer: intra-scale operations (within each graph view) and, crucially, inter-scale operations that model the distinct operator orderings (across different graph views). This yields an interpretable and physically grounded mechanism that links local topology to global effects. We prove that under mild assumptions, this heterogeneous message-passing architecture is more expressive than single-graph models and guaranteed to achieve a lower or equal oracle prediction risk.

We instantiate our framework on two scientific domains. For molecular property prediction on the QM7b dataset, we develop GINE-Gt, a specialized architecture that uses the powerful Graph Isomorphism Network with Edge Features(GINE)~\cite{hu2019strategies} as the backbone. Second, for cosmological parameter inference from point cloud data, we develop EGNN-Gt based on Equivariant Graph Neural Network (EGNN)~\cite{satorras2021n}, a powerful GNN architecture that guarantees equivariance to rotations, translations, and reflections.

The empirical results demonstrate the efficacy of our framework. In QM7b, GINE-Gt outperforms invariant-feature models and a suite of single-graph GNN baselines in most prediction targets. In the cosmological simulations from the CAMELS suite, EGNN-Gt demonstrates superior overall performance over its corresponding single-graph counterparts across a wide range of interaction radii. These results not only highlights the power of our multi-view strategy but also demonstrates the potential of our multi-view graph tuple framework for a broader range of applications involving continuous relational data. The source code is available on Github\footnote{https://github.com/chenshy202/Multi-View-Graph-Learning/tree/main}.

\section{Related Work}
Our work lies at the intersection of heterogeneous graph learning that process graphs with different typed nodes or edges, and multi-view representation learning to extract features from different scales. 

\textbf{Heterogeneous graph learning.} Early heterogeneous GNNs such as R-GCN~\cite{schlichtkrull2018modeling} and HAN~\cite{wang2019heterogeneous} demonstrated the benefit of relation-specific message-passing design, but they assume pre-defined, \emph{discrete} node and edge types (e.g., knowledge graphs or bibliographic networks). To go beyond pre-defined relations, Graph Transformer Networks (GTN)~\cite{yun2019graph} proposed to softly select relation-specific adjacency matrices and then generate new graphs by their matrix products. We generalize this heterogeneous graph learning paradigm to homogeneous graphs, by inducing different relations via partitioning \emph{continuous} edge features (e.g., physical distances or chemical interactions). 

\textbf{Multi-view graph learning.} A recent line of work constructs multiple relational views from a single graph, motivated from self-supervising learning 
(e.g., contrastive multi-view learning \cite{hassani2020contrastive}) or community detection (e.g., variational edge partition model \cite{he2022variational}). 
In contrast to these prior works, we are motivated to construct multiple views based on a physical measure of interaction strength from scientific applications (e.g., the Coulomb matrix in the molecular domain, and the Euclidean distance matrix in the cosmological applications). 

\textbf{Multi-scale GNNs.} Our approach is architecturally most related to multi-scale GNNs that learn a hierarchical representation from a graph, such as FraGAT~\cite{zhang2021fragat} designed for molecular property prediction,
and MultiScale MeshGraphNets~\cite{fortunato2022multiscale} for physics simulation. 
But these methods typically build multiple graphs on different node sets (e.g., atoms vs. fragments and fine vs. coarse mesh), which requires non-trivial cross-level alignment and transfer operators. In contrast, our framework defines multiple graphs over the same node set by partitioning a continuous interaction strength, enabling simple and interpretable within- and between-graph message passing.

\section{Preliminaries}
In this work, we consider graphs denoted as $\mathcal{G} = (\mathcal{V}, \mathcal{E})$, where $\mathcal{V}=\{v_1, \dots, v_n\}$ is a set of $n$ nodes and $\mathcal{E} \subseteq \mathcal{V} \times \mathcal{V}$ is a set of edges. Each node $v_i$ is associated with an initial feature vector, and each edge $(i,j)$ with an initial feature vector. Before being processed by any network layers, initial node and edge features are projected into a hidden space via learnable encoders. For notational simplicity, we let $h_i^{(l)}$ denote the encoded feature vector for node $i$ at layer $l$, and we let $e_{ij}$ denote the encoded feature for the edge $(i,j)$. We collectively represent all node features at a given layer $l$ as a matrix $H^{(l)}$ and all edge features as a matrix $E$.


\subsection{Graph Isomorphism Network with Edge Features (GINE)}
GINE~\cite{hu2019strategies} extends Graph Isomorphism Network (GIN) by incorporating edge features into its message passing procedure:
\begin{equation}
h_i^{(l+1)} =MLP^{(l)}\!\Bigl((1+\varepsilon)\,h_i^{(l)} \;+\!\! \sum_{j\in\mathcal N(i)} \text{ReLU}\left(h_j^{(l)} + e_{ij}\right)\Bigr)
\end{equation}
The entire single-layer update process can then be concisely expressed as:
\begin{equation}
    H^{(l+1)} = \text{GINEConv}\left(H^{(l)}, \mathcal{E}, E\right)
    \label{eq:gine_operator}
\end{equation}

\subsection{Equivariant Graph Neural Networks (EGNN)}

Equivariant Graph Neural Networks (EGNNs) \cite{satorras2021n} incorporate geometric information by endowing each node with Euclidean coordinates $x_i \in \mathbb{R}^d$. They jointly update node features and coordinates in a way that is equivariant to node permutations and Euclidean isometries, i.e., translations, rotations (and reflection). An EGNN convolution layer (EGCL) is defined as follows:
\begin{align}
     &m_{ij} = \phi_e\left(h_i^{(l)},\, h_j^{(l)},\, \|x_i^{(l)}-x_j^{(l)}\|^2,\, a_{ij}\right), \\
    &x_i^{(l+1)} = x_i^{(l)} + C \sum_{j\in\mathcal{N}(i)} (x_i^{(l)}-x_j^{(l)})\,\phi_x(m_{ij}), \\
    &h_i^{(l+1)} = \phi_h\left(h_i^{(l)},\, \sum_{j\in\mathcal{N}(i)} m_{ij}\right).
\end{align}
Here, $\phi_e, \phi_x, \phi_h$ are learnable functions (e.g., MLPs), $C$ is a scalar and $a_{ij}$ represents optional edge attributes. In this work, we simply use edge features, i.e., $a_{ij} = e_{ij}$. We denote the computation of the EGNN convolution layer as
\begin{equation}
    (H^{(l+1)}, X^{(l+1)}) = \mathrm{EGCL}(H^{(l)}, X^{(l)}, \mathcal{E}, E).
\end{equation}

\begin{figure}[htbp]
    \centering
    \includegraphics[width=0.95\textwidth]{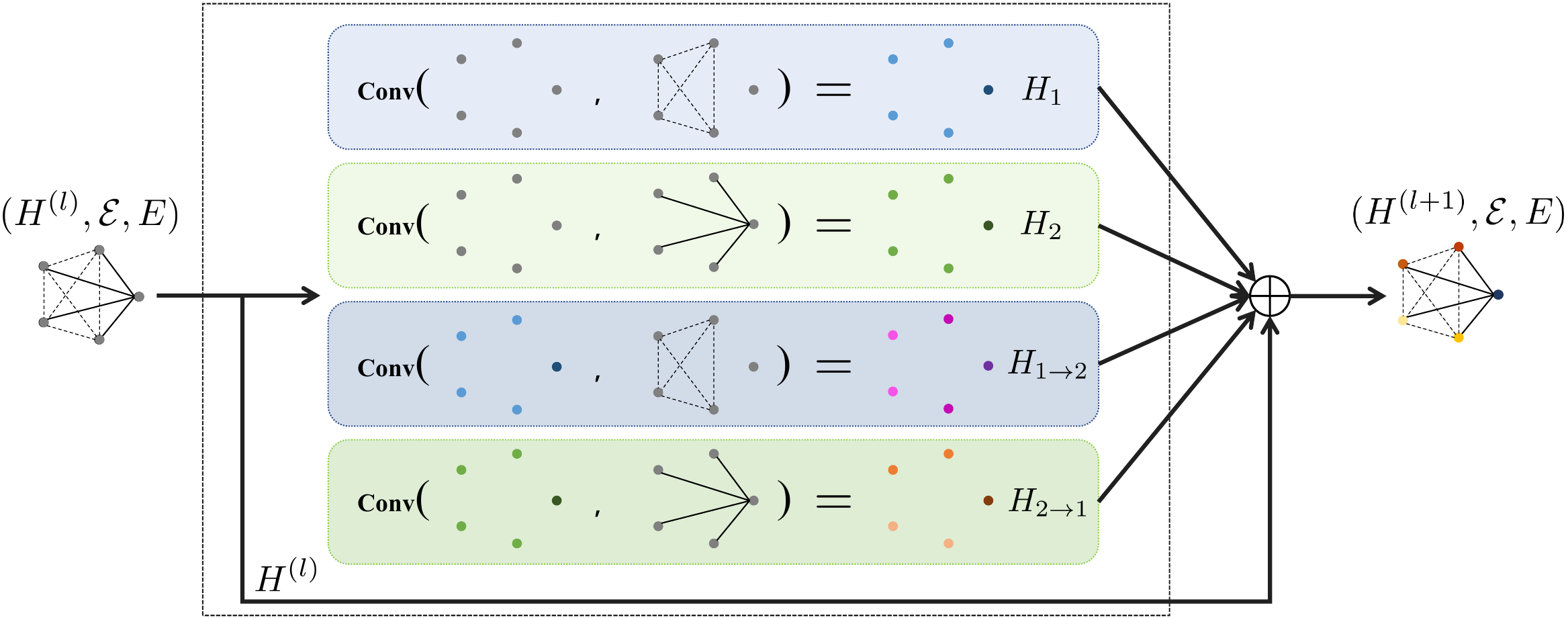} 
    \caption{An illustration of our heterogeneous message-passing architecture for $k=2$ views. The node embeddings $H_1, H_2$ are obtained from message-passing within each graph (with the corresponding edges and edge features); The node embeddings $H_{1 \to 2}, H_{2 \to 1}$ are computed from message-passing across views, one for each direction. These embeddings are then aggregated per ~\ref{eq:gt_update}.}
    \label{fig:framework} 
\end{figure}

\section{Method}
Our work introduces a multi-view graph-tuple framework for learning from complex relational systems where interactions occur at different scales. The core principle is to model these interactions and learn how information flows both within and between these scales. This is achieved by decomposing a graph $\mathcal{G}$ into a graph tuple, $(\mathcal{G}_1, \mathcal{G}_2, \ldots, \mathcal{G}_k)$, that provides distinct yet complementary views of the interaction space, representing fine-grained local structures and broader contextual relationships. This graph-tuple framework can be instantiated using different graph neural network backbones. We present two such instantiations: an edge-aware model for attributed graphs, GINE-Gt based on~\cite{hu2019strategies}, and an equivariant extension for geometric data, EGNN-Gt~\cite{satorras2021n}. 
\subsection{Multi-view Graph-tuple Representation}

We begin by decomposing a single graph into multiple views (subgraphs), represented as a graph tuple $\mathcal{G} = (\mathcal{G}_1, \ldots, \mathcal{G}_k)$. Each graph $\mathcal{G}_i = (\mathcal{V}, \mathcal{E}_i)$ is defined on the same node set, but with disjoint edges sets to capture different interaction scales, namely $\bigcup_{i=1}^k \mathcal{E}_i = \mathcal{E}$ and $\mathcal{E}_i \bigcap \mathcal{E}_j = \emptyset $ for $i\neq j$. For example, when $k=2$, we can decompose $\mathcal{G}$ (with nonnegative edge weights) into a strong-connection graph $\mathcal{G}_1$ containing the largest-magnitude edges (greater than certain threshold $\tau$), and a weak-connection graph $\mathcal{G}_2$ capturing the remaining connections. 
The precise instantiation of these graphs is detailed in Section~\ref{sec: Molecular Property Prediction} and~\ref{sec: Cosmological Parameter Inference}.

\subsection{Heterogeneous Message-Passing Architecture}

To learn the multi-view representations from a graph tuple, we use a heterogeneous message-passing architecture, where each layer updates node representations by computing and integrating information from multiple distinct message-passing operations. The final update is defined as a residual combination of these multi-view representations, governed by learnable scalar weights $c_k$:
\begin{equation}
\label{eq:gt_update}
H^{(l+1)} = H^{(l)} + \sum_{i=1}^k c_i \cdot H_{i} + \sum_{i\neq j} \left( c_{ij} \cdot H_{i \to j} + c_{ji} \cdot H_{j \to i} \right), 
\end{equation}
where $H_i$ denotes the node embeddings from intra-scale message-passing within each graph $\mathcal{G}_i$, and $H_{i \to j}, H_{j \to i}$ denote the node embeddings from inter-scale message-passing across different graphs $\mathcal{G}_i, \mathcal{G}_j$. 
These representations are designed to capture distinct patterns. For example, given the graph-tuple with a strong-connection graph $\mathcal{G}_1$ and a weak-connection graph $ \mathcal{G}_2$, the intra-scale representations ($H_1, H_2$) extract local interactions from $\mathcal{G}_1$ and broader, long-range information from $\mathcal{G}_2$, respectively. Then these representations are fused in $H_{i \to j}, H_{j \to i}$ via inter-scale message-passing, which capture relational information that is sensitive to the order of operations. 

Figure~\ref{fig:framework} provides an overview of our framework. 
We note that if the original edge set $\mathcal{E}$ contains \emph{typed} edges, it can be naturally partitioned into homogeneous edge sets, one for each type, as done in R-GCN \cite{schlichtkrull2018modeling}. Our framework extends R-GCN from heterogeneous graphs to homogeneous graphs, replacing discrete edge types with partitions induced by continuous edge features. The architectures proposed below also extend the Graph Tuple Neural Network framework from \cite{velasco2024graph}.

\subsubsection{GINE-Gt}

For general graphs with node features and edge attributes, we implement our framework using the GINE layer \cite{hu2019strategies}. The intra-scale message-passing for $i=1, \ldots, k$ are computed as
\begin{equation}
   H_{i} = \text{GINEConv}_i\left(H^{(l)}, \mathcal{E}_i, E_i\right) .
\end{equation}
The inter-scale message-passing are then computed by
\begin{equation}
   H_{i \to j} = \text{GINEConv}_{ij}\left(H_{i}, \mathcal{E}_j, E_{j}\right), \quad   H_{j \to i} = \text{GINEConv}_{ji}\left(H_{j}, \mathcal{E}_i, E_{i}\right),
\end{equation}
where $i\neq j$.
Each $\text{GINEConv}_k$ is a distinct function with its own parameter weights, allowing the model to learn specialized functions for each interaction type.

\subsubsection{EGNN-Gt}

For geometric data where node features represent point positions in $\mathbb R^d$, we provide an $E(d)$-equivariant implementation of our framework using the EGCL layer \cite{satorras2021n}. The overall feature update follows Eq.~\ref{eq:gt_update}, while the position feature update is analogously defined as:
\begin{equation}
\label{eq:x_update}
X^{(l+1)} = X^{(l)} + \sum_{i=1}^k c_i \Delta X_{i} + \sum_{i\neq j}\left( c_{ij} \Delta X_{i \to j} + c_{ji} \Delta X_{j \to i} \right).
\end{equation}
The representations $(H_k, \Delta X_k)$, which contain both feature updates and coordinate displacements, are all computed from a single, shared EGCL layer.


The intra-scale representations are computed as
\begin{equation}
(H_i, \Delta X_i) = \text{EGCL}\!\left(H^{(l)}, X^{(l)}, \mathcal{E}_i, E_{i}\right), \quad \text{for } i = 1, \ldots, k.
\end{equation}
Subsequently, the inter-scale representations are obtained using these intermediate outputs: for $i \neq j$,
\begin{align}
    (H_{i \to j}, \Delta X_{i \to j}) &= \text{EGCL}\left(H_i, X^{(l)} + \Delta X_i, \mathcal{E}_j, E_{j}\right); \\
    (H_{j \to i}, \Delta X_{j \to i}) &= \text{EGCL}\left(H_j, X^{(l)} + \Delta X_j, \mathcal{E}_i, E_{i}\right).
\end{align}
This formulation allows the EGNN-Gt layer to learn geometrically-aware representations from the multi-view interaction patterns.

\section{Expressivity}
\label{sec:theory_main}

We analyze our multi-view graph-tuple framework in a simplified linear setting to establish its expressivity and generalization properties. We consider $k=2$ and define the shift operators $S_1$ and $S_2$ to be the adjacency matrices $\mathcal{G}_1$ and $\mathcal{G}_2$ respectively.
We study three classes of linear graph filters: $H_1(m)$ the polynomials of degree $m$ in $S_1$, $H_0(m)$ the polynomials of degree $m$ in $S_1+S_2$ the adjacency of the dense graph $\mathcal G$, and our multi-view graph-tuple class $H_{\mathrm{Gt}}(m)$ of multivariate polynomials of degree $m$ in $(S_1,S_2)$. See Definition \ref{def:classes} in Appendix \ref{app:theory_details}. We show the following (see proofs in Appendix \ref{app:theory_details}):

\begin{informaltheorem}[Expressivity]\label{inf:expr}
For any degree bound $m$, $H_1(m)\subseteq H_{\mathrm{Gt}}(m)$ and $H_0(m)\subseteq H_{\mathrm{Gt}}(m)$; if $S_1S_2\neq S_2 S_1$ and $m\ge2$, then the latter inclusion is strict.
See Proposition~\ref{thm:expressivity} in Appendix~\ref{app:theory_details}.
\end{informaltheorem}

\begin{informaltheorem}[Oracle risk dominance]\label{inf:risk}
For any $m$, $\inf_{g\in H_{\mathrm{Gt}}(m)}R(g)\le \inf_{q\in H_0(m)}R(q)$ and
$\inf_{g\in H_{\mathrm{Gt}}(m)}R(g)\le \inf_{p\in H_1(m)}R(p)$. Moreover, if the oracle predictor $M^\star$ lies outside the expressivity of the baseline class $H_0(m)$, the advantage is strict, and the performance gap is a strictly positive, quantifiable value. See Proposition~\ref{thm:quantitative_gap} in Appendix~\ref{app:theory_details}.
\end{informaltheorem}

\begin{proof}[Proof Sketch]
It is easy to see that the single operator baseline models correspond to a special case of the multi-view graph-tuple class for a specific choice of coefficients. The risk dominance is a direct consequence of this expressivity gap.
\end{proof}

\section{Molecular Property Prediction}
\label{sec: Molecular Property Prediction}
\begin{table*}[htbp]
\centering
\fontsize{7.5pt}{9.0pt}\selectfont
\setlength{\tabcolsep}{3pt}
\caption{Performance comparison of our GINE-Gt with all baselines on the QM7b dataset. The result report the Mean Absolute Error (MAE) $\pm$ standard error over ten folds (lower is better). The best result in each column is highlighted in \textbf{bold}, and the second-best is in \textit{italics}. Our GINE-Gt is the top-performing method overall, while GINE-2 is the strongest among the single-graph baselines.}
\label{tab:mole_results}
\renewcommand{\arraystretch}{0.8} 
\begin{tabular}{l ccccccc} 
\toprule
\textbf{MAE $\downarrow$} & Atomization & Excitation & Absorption & HOMO & LUMO & 1st excitation & Ionization \\
 & PBE0 & ZINDO & ZINDO & ZINDO & ZINDO & ZINDO & ZINDO \\
\midrule
KRR \cite{wu2018moleculenet}  & 9.3 & 1.83 & 0.098 & 0.369 & 0.361 & 0.479 & 0.408 \\
DS-CI & $12.849 \pm 0.757$ & $1.776 \pm 0.069$ & $0.086 \pm 0.003$ & $0.401 \pm 0.017$ & $0.338 \pm 0.048$ & $0.492 \pm 0.058$ & $0.422 \pm 0.012$ \\
DTNN \cite{wu2018moleculenet}  & 21.5 & 1.26 & 0.074 & 0.192 & 0.159 & 0.296 & 0.214 \\
DS-CI+ & $\textit{7.650} \pm 0.399$ & $1.045 \pm 0.030$ & $0.069 \pm 0.005$ & $0.172 \pm 0.009$ & $0.119 \pm 0.005$ & $0.160 \pm 0.011$ & $0.189 \pm 0.011$ \\
GINE-0     & $12.812 \pm 0.372$ & $1.034 \pm 0.027$ & $0.064 \pm 0.002$ & $0.197 \pm 0.007$ & $0.072 \pm 0.002$ & $0.143 \pm 0.003$ & $0.212 \pm 0.005$ \\
GINE-0.5   & $12.171 \pm 0.543$ & $1.030 \pm 0.016$ & $0.068 \pm 0.002$ & $0.207 \pm 0.004$ & $0.080 \pm 0.002$ & $0.143 \pm 0.006$ & $0.240 \pm 0.007$ \\
GINE-1     & $11.170 \pm 0.337$ & $1.000 \pm 0.015$ & $\textit{0.064} \pm 0.001$ & $0.177 \pm 0.005$ & $0.093 \pm 0.005$ & $0.120 \pm 0.004$ & $0.200 \pm 0.005$ \\
GINE-2     & $10.349 \pm 0.590$ & $0.998 \pm 0.019$ & $0.067 \pm 0.002$ & $\textit{0.147} \pm 0.004$ & $\textbf{0.063} \pm 0.001$ & $\textit{0.116} \pm 0.006$ & $\textit{0.176} \pm 0.009$ \\
GINE-2.5   & $11.306 \pm 0.677$ & $\textit{0.969} \pm 0.013$ & $0.067 \pm 0.001$ & $0.168 \pm 0.006$ & $\textit{0.066} \pm 0.002$ & $0.131 \pm 0.004$ & $0.193 \pm 0.005$ \\
GINE-Gt & $\textbf{6.700} \pm 0.183$ & $\textbf{0.955} \pm 0.011$ & $\textbf{0.062} \pm 0.001$ & $\textbf{0.131} \pm 0.005$ & $0.067 \pm 0.001$ & $\textbf{0.111} \pm 0.003$ & $\textbf{0.151} \pm 0.005$ \\
\midrule
\textbf{MAE $\downarrow$} & Affinity & HOMO & LUMO & HOMO & LUMO & Polarizability & Polarizability \\
 & ZINDO & KS & KS & GW & GW & PBE0 & SCS \\
\midrule
KRR \cite{wu2018moleculenet}  & 0.404 & 0.272  & 0.239  & 0.294  & 0.236  & 0.225  & 0.116 \\
DS-CI & $0.404 \pm 0.047$ & $0.302 \pm 0.009$ & $0.225 \pm 0.010$ & $0.329 \pm 0.016$ & $0.213 \pm 0.008$ & $0.255 \pm 0.015$ & $0.114 \pm 0.008$ \\
DTNN \cite{wu2018moleculenet} & 0.174& \textit{0.155}  & 0.129  & 0.166  & 0.139  & 0.173  & 0.149 \\
DS-CI+& $\textit{0.122} \pm 0.002$ & $0.169 \pm 0.007$ & $0.135 \pm 0.007$ & $0.183 \pm 0.005$ & $0.139 \pm 0.004$ & $0.139 \pm 0.005$ & $\textit{0.088} \pm 0.004$ \\
GINE-0     & $0.082 \pm 0.002$ & $0.184 \pm 0.008$ & $0.109 \pm 0.005$ & $0.198 \pm 0.008$ & $0.116 \pm 0.004$ & $0.170 \pm 0.006$ & $0.094 \pm 0.003$ \\
GINE-0.5   & $0.087 \pm 0.002$ & $0.207 \pm 0.007$ & $0.103 \pm 0.003$ & $0.234 \pm 0.010$ & $0.129 \pm 0.007$ & $0.189 \pm 0.004$ & $0.102 \pm 0.002$ \\
GINE-1     & $0.088 \pm 0.003$ & $0.176 \pm 0.005$ & $0.096 \pm 0.002$ & $0.201 \pm 0.004$ & $0.118 \pm 0.003$ & $0.171 \pm 0.004$ & $0.104 \pm 0.003$ \\
GINE-2     & $\textbf{0.067} \pm 0.002$ & $0.163 \pm 0.006$ & $\textbf{0.080} \pm 0.002$ & $\textit{0.166} \pm 0.003$ & $\textit{0.106} \pm 0.002$ & $\textit{0.135} \pm 0.003$ & $0.092 \pm 0.005$ \\
GINE-2.5   & $\textit{0.070} \pm 0.002$ & $0.162 \pm 0.006$ & $0.086 \pm 0.004$ & $0.180 \pm 0.005$ & $0.112 \pm 0.002$ & $0.142 \pm 0.004$ & $\textit{0.087} \pm 0.003$ \\
GINE-Gt & $0.073 \pm 0.002$ & $\textbf{0.133} \pm 0.002$ & $\textit{0.084} \pm 0.001$ & $\textbf{0.148} \pm 0.003$ & $\textbf{0.101} \pm 0.002$ & $\textbf{0.098} \pm 0.002$ & $\textbf{0.071} \pm 0.002$  \\
\bottomrule
\end{tabular}
\end{table*}

We consider the QM7b benchmark~\cite{blum2009970, montavon2013machine}, which contains $7\,211$ molecules with $14$ regression targets.  
Each molecule is encoded by a $n\times n$ Coulomb matrix $X$ whose entries depend only on nuclear charges $Z_i\!\in\!\mathbb{R}$ and 3D coordinates $R_i\!\in\!\mathbb{R}^3$:
\begin{equation}
    X_{ij}
=
\begin{cases}
0.5\,Z_i^{2.4}, & i=j,\\[2pt]
\dfrac{Z_i Z_j}{\lVert R_i - R_j\rVert}, & i\neq j,
\end{cases}
\end{equation}
Then we build the graph $\mathcal{G} = (\mathcal{V}, \mathcal{E})$ with $\mathcal{V}=\{1,\dots,n\}$ and $\mathcal{E}=\{(i,j)\mid i\neq j\}$, i.e.\ all atom pairs are connected while self–loops $(i,i)$ are removed. Notably, since its off-diagonal entries are computed from inter-atomic distances, the Coulomb matrix is invariant to rotations and translations (i.e., E(3)-invariant) by construction.

\subsection{Graph Construction}
Graphs are constructed from the molecule's Coulomb matrix. The baseline models, denoted GINE-$c$, operate on a single graph formed by applying a threshold $c$, where edges are all pairs $(i,j)$ with an interaction strength $X_{ij} \geq c$. This method discards all information below the threshold.

Our main model, GINE-Gt, operates on a multi-view graph tuple $(\mathcal{G}_1, \mathcal G_2)$ derived by partitioning the interaction space at a boundary of $c=2$ selected over a validation set (see Table~\ref{tab:mole_results}). The strong-connection graph ($\mathcal G_1$) is thus composed of edges where $X_{ij} \geq 2$, while the weak-connection graph ($\mathcal G_2$) comprises all remaining edges. This threshold effectively identifies the primary interaction backbone for the strong-connection graph ($\mathcal G_1$) while assigning the remaining contextual interactions to the weak-connection graph ($\mathcal G_2$).

The percentage of retained edges, along with full implementation details such as feature construction, model configurations, and training protocols, are provided in Appendix~\ref{app:Molecular Property Prediction}.

\subsection{Results and Analysis} \label{subsec: Molecular Property Prediction}
Table~\ref{tab:mole_results} presents the performance comparison of our GINE-Gt model against multiple baselines. These include a series of single-graph GINE-$c$ models as well as several non-graph-based methods: Kernel Ridge Regression (KRR), Deep Tensor Neural Network (DTNN)~\cite{schutt2017quantum}, and the state-of-the-art invariant feature model, DS-CI+~\cite{blum2024learning}. Results for KRR and DTNN are taken from prior work~\cite{wu2018moleculenet}. 

Our proposed GINE-Gt model is the top-performing method, achieving the best Mean Absolute Error (MAE) on 11 of the 14 prediction targets.
The best single-graph model, GINE-2, demonstrates the importance of focusing on the strong interactions compared to the full-graph GINE-0. However, our results suggest that weak interactions are also relevant.

While our multi-view approach GINE-Gt outperforms these baselines, we note that GINE-2 achieves better performance on three targets. This suggests that these specific properties are predominantly governed by strong, short-range interactions. In such cases, the global context provided by weaker interactions offers limited benefit and may introduce a small amount of non-essential information. This highlights the potential for a more adaptive partitioning mechanism that a flexible, learnable threshold, rather than our current fixed one, could allow the model to dynamically balance the two views and further improve the multi-view graph-tuple framework's performance.



\begin{figure}[t]
  \centering
  \begin{subfigure}[t]{0.49\textwidth}
    \centering
    \includegraphics[width=1\linewidth]{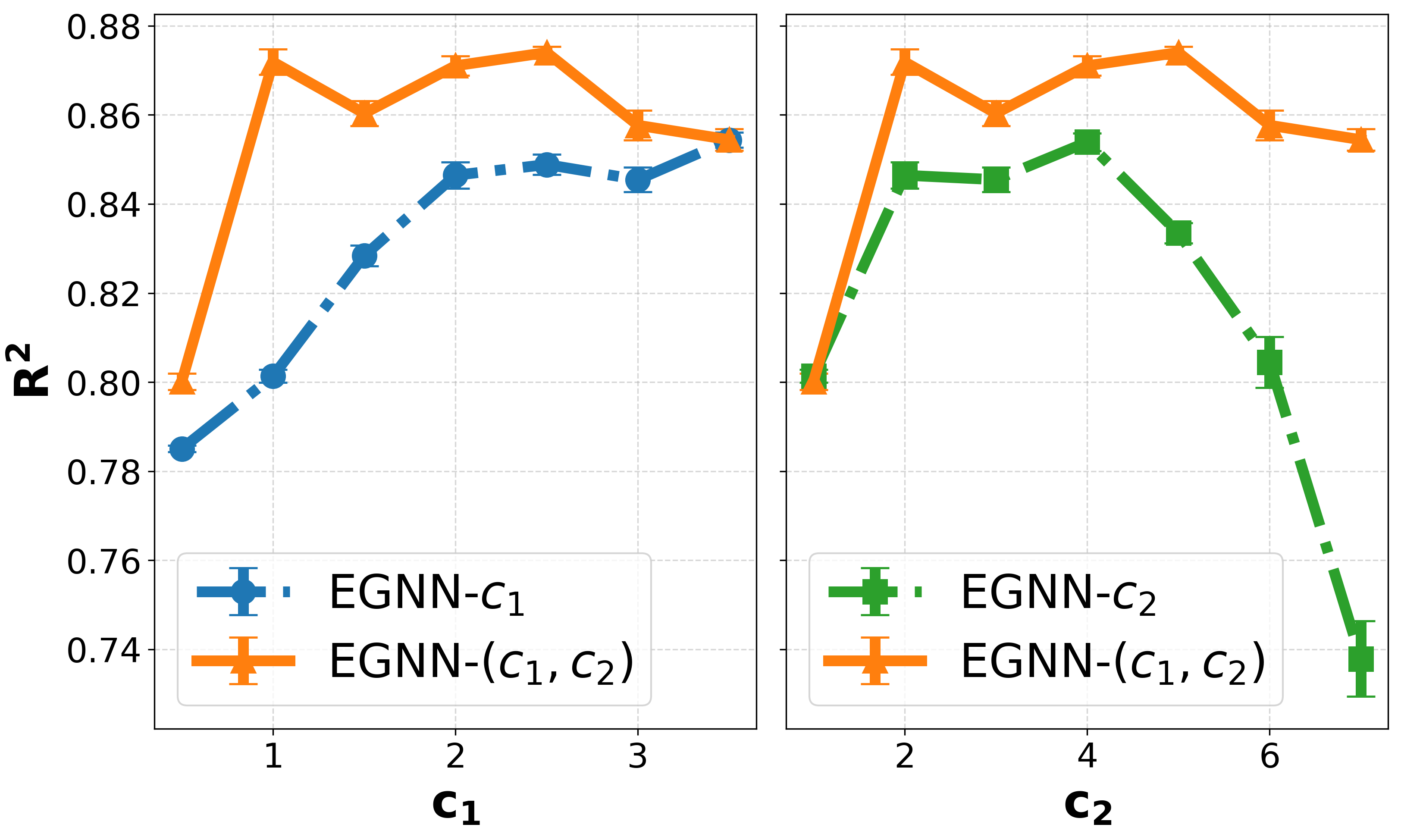}
    \vspace{-1.6em}
    \caption{CAMELS: $\Omega_{\rm m}$}
    \label{fig:camels-om}
  \end{subfigure}
  \hfill
  \begin{subfigure}[t]{0.49\textwidth}
    \centering
    \includegraphics[width=1\linewidth]{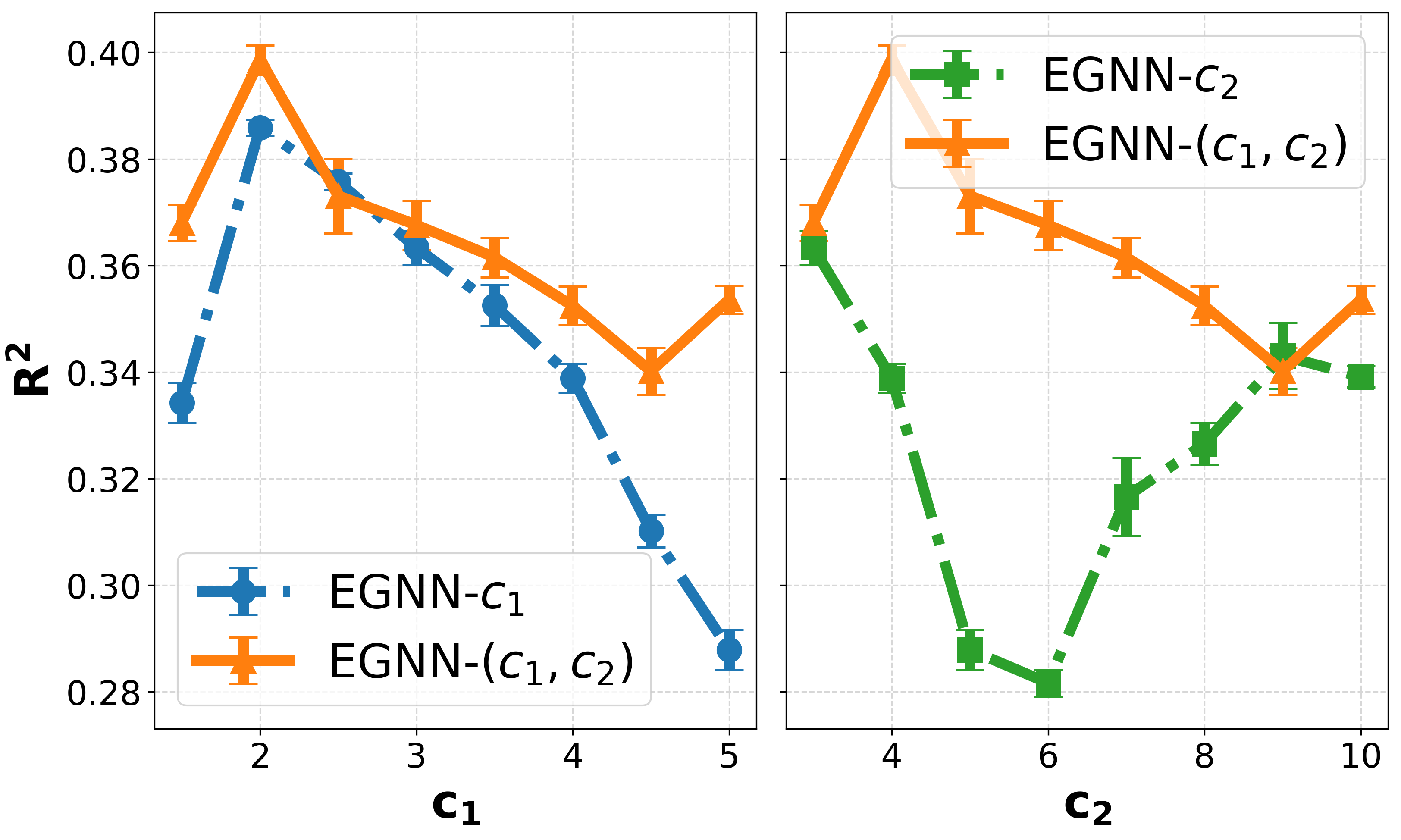}
    \vspace{-1.6em}
    \caption{CAMELS: $\sigma_8$}
    \label{fig:camels-s8}
  \end{subfigure}

\vspace{0.8em}
  \centering
  \begin{subfigure}[t]{0.49\textwidth}
    \centering
    \includegraphics[width=1\linewidth]{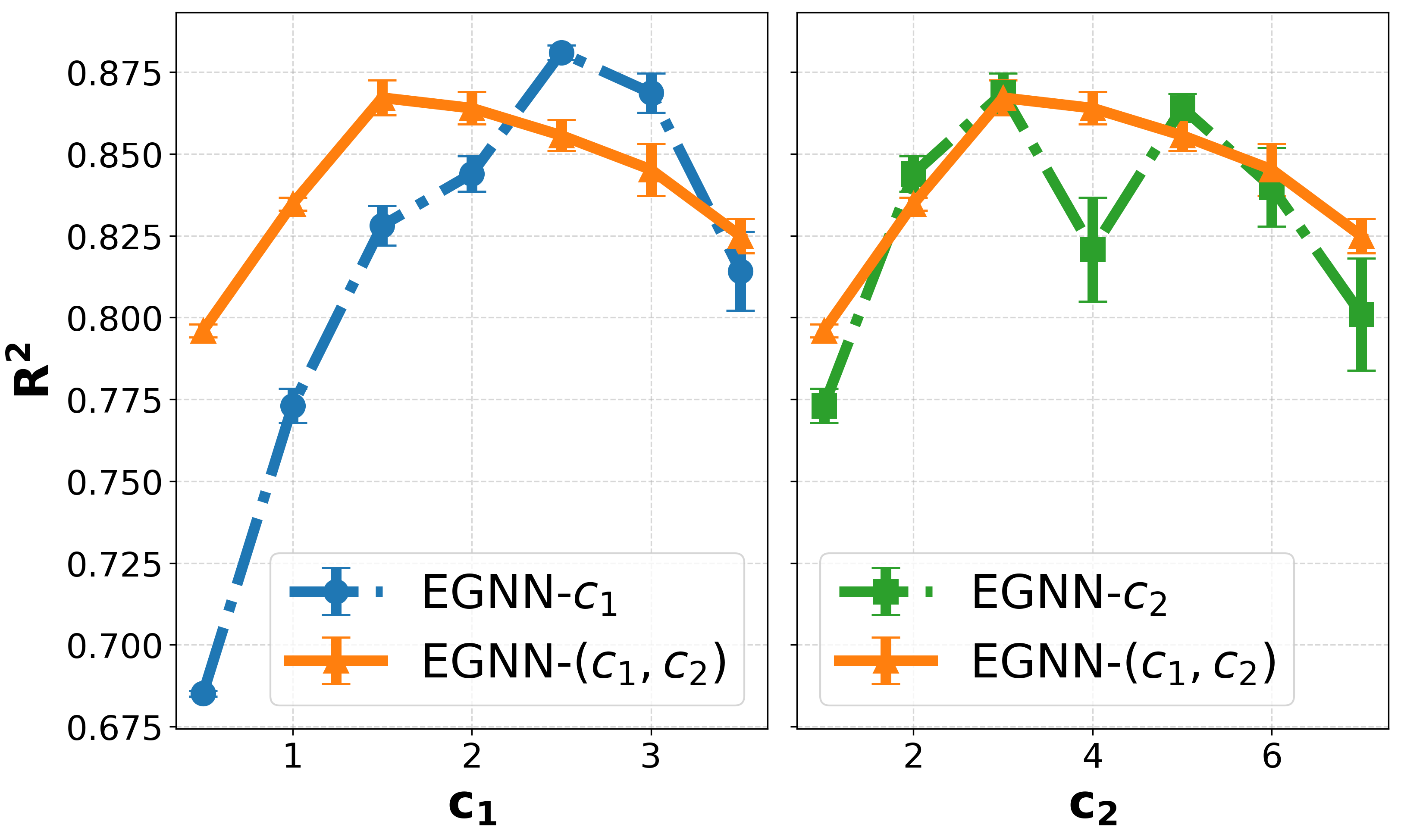}
    \vspace{-1.6em}
    \caption{CAMELS-SAM: $\Omega_{\rm m}$}
    \label{fig:camels-sam-om}
  \end{subfigure}
  \hfill
  \begin{subfigure}[t]{0.49\textwidth}
    \centering
    \includegraphics[width=1\linewidth]{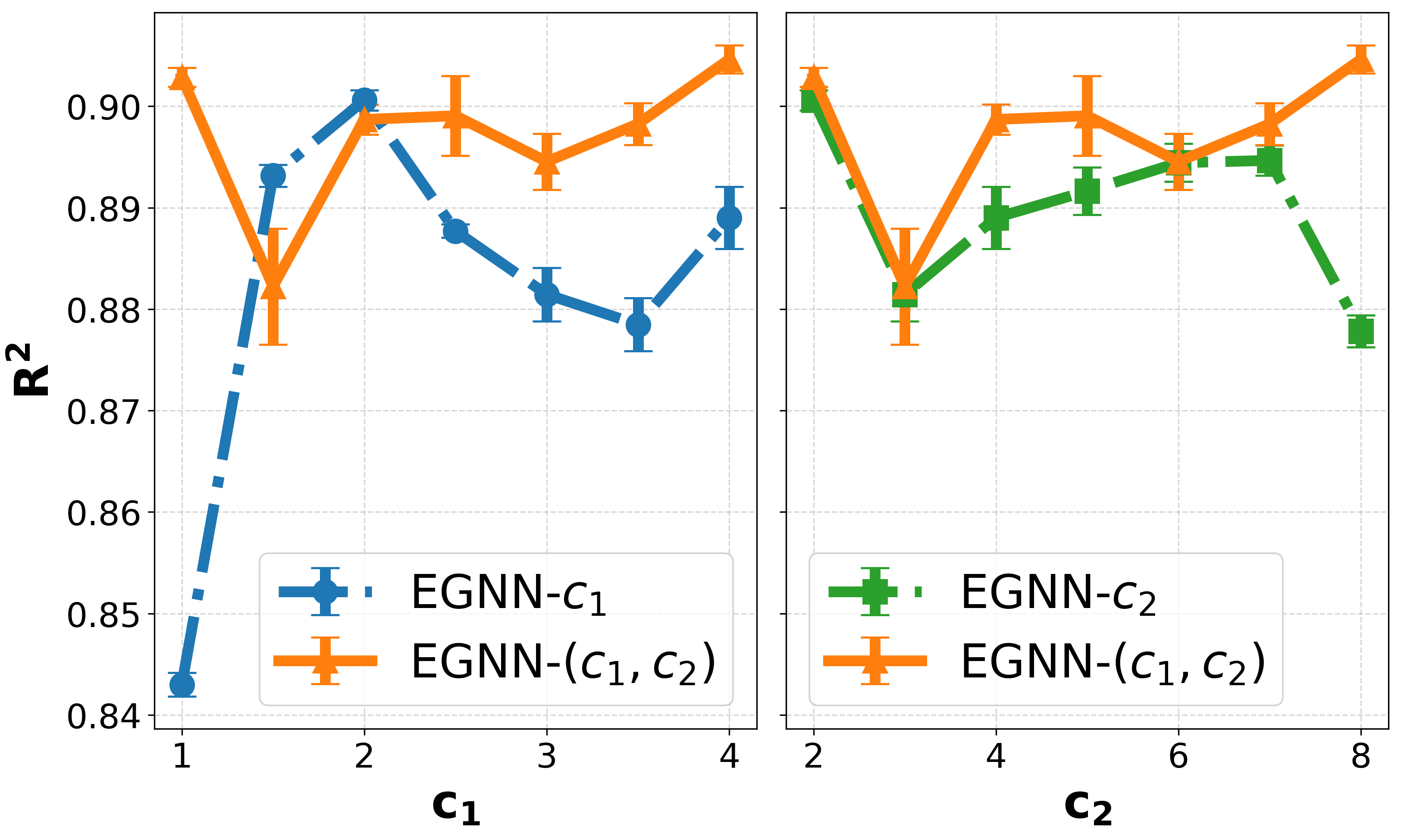}
    \vspace{-1.6em}
    \caption{CAMELS-SAM: $\sigma_8$}
    \label{fig:camels-sam-s8}
  \end{subfigure}
  
  \caption{Performance comparison on the CAMELS and CAMELS-SAM datasets for cosmological parameter prediction. We plot the coefficient of determination ($R^2$, higher is better) for our multi-view EGNN-Gt model against two single-graph baselines. The left plot in each subfigure compares against the strong-connection baseline (EGNN-$c_1$), while the right plot compares against the dense-graph baseline (EGNN-$c_2$). All results are reported as the mean and standard error over 10 runs.}
  \label{fig:cosmoresults}
\end{figure}

\section{Cosmological Parameter Prediction}
\label{sec: Cosmological Parameter Inference}
We evaluate our EGNN-Gt model on the CosmoBench benchmark~\cite{huang2025cosmobench}, specifically using the CAMELS (TNG) and the CAMELS-SAM cosmological point cloud datasets. Each sample in these datasets is a cosmological simulation cloud where nodes represent dark matter halos or galaxies. The model input is the matrix of 3D halo (galaxy) positions representing their present-day configuration, $X \in \mathbb{R}^{N \times 3}$, where $N$ is the number of halos in the point cloud. 

The primary task is a cloud-level regression problem, where the model infer from the present-day positions $X$ about the cosmological parameters $y = (\Omega_m, \sigma_8)$ that control the evolution of halos (galaxies). The performance on this task is measured using the coefficient of determination ($R^2$), defined as:
\begin{equation}
\label{eq:r_squared}
R^2 = 1 - \frac{\sum_{i=1}^{n_{\text{test}}} (f(X_i) - y_i)^2}{\sum_{i=1}^{n_{\text{test}}} (\bar{y} - y_i)^2},
\end{equation}
where $R^2$ evaluates the cosmology parameter prediction and higher $R^2$ indicates better model fit. 
\subsection{Graph Construction}
Insights from prior work~\cite{huang2025cosmobench} and our own preliminary experiments suggest that informative halo interactions typically occur within a radius range approximately from 0 to 10~Mpc/h. To establish strong baselines, we evaluated two types of single-graph models. First, we form a series of strong-connection baseline, EGNN-$c_1$ by connecting halos within a cutoff radius $c_1$. Second, we create dense-graph baselines, EGNN-$c_2$ using a larger radius $c_2 > c_1$. For our EGNN-Gt model, we implement the multi-view approach by partitioning the interaction space with two radii: a strong-connection radius $c_1$ and a weak-connection radius $c_2$. To explore the benefits of multi-view processing while maintaining a simple and interpretable relationship between the two scales, we enforce a fixed ratio by setting $c_2 = 2c_1$ in all our experiments. This allows the model to simultaneously capture immediate local neighborhoods ($\mathcal{G}_1$) and broader, second-order contextual regions ($\mathcal{G}_2$). The specific ranges for the systematic search over $c_1$ for each task, along with full implementation details, are provided in Appendix~\ref{app:Cosmological Parameter Inference}.

\subsection{Results and Analysis}\label{subsec: Cosmological Parameter Prediction}
Using the search ranges for $c_1$ established in our experimental setup, we present the performance of our multi-view EGNN-Gt model in Figure~\ref{fig:cosmoresults}, with all results shown as mean $R^2$ values with standard errors. Each row corresponds to a dataset, and each column to a target cosmological parameter. The plots in each subfigure provide a direct comparison: the left plot shows the performance of our multi-view model, EGNN-($c_1, c_2$), against its corresponding strong-connection baseline, EGNN-$c_1$, as a function of the radius $c_1$. The right plot shows a similar comparison against the dense-graph baseline EGNN-$c_2$ as a function of the radius $c_2$.

Across both datasets and target parameters, our multi-view EGNN-Gt model achieve better performance than single-graph baselines in most cases. Specifically, on the CAMELS dataset (Figures~\ref{fig:camels-om} and~\ref{fig:camels-s8}), our multi-view method EGNN-Gt outperforms the corresponding single-graph baselines, i.e., strong-connection and dense-graph baselines (blue and green lines) at nearly every tested radius, with performance being, at worst, comparable in very few instances.

On the CAMELS-SAM dataset (Figures~\ref{fig:camels-sam-om} and~\ref{fig:camels-sam-s8}), our EGNN-Gt model still exhibits better performance at most tested radius. However, there are instances where the single-graph baselines achieve better results. This does not indicate a failure of the multi-view method. Instead, we attribute this to the constraint of our fixed ratio, $c_2=2c_1$. The number of edges in a radius graph grows non-linearly with the radius (approximately as $R^3$ in 3D space). Therefore, a simple linear scaling between the two radii may not always capture the optimal balance of information density from the strong- and weak-connection graphs. This suggests that exploring adaptive or non-linear relationships between $c_1$ and $c_2$ is a promising direction for future work.

A second key design parameter is the number of partitions in the graph-tuple. In this work we intentionally instantiate the graph-tuple with $k = 2$ views corresponding to the strong and weak graphs. This two-scale design already captures the main separation between primary and contextual interactions in our dense graphs, while keeping the number of intra- and inter-scale message-passing paths manageable. Allowing more than two partitions would increase computational and tuning complexity by introducing additional operators and paths, but it could provide greater flexibility to capture diverse interactions. Beyond multi-scale analysis, our framework is also naturally suited to heterogeneous graphs, where each of the $k$ views may correspond to a distinct relation type. Systematically exploring these higher-order, multi-relational graph-tuples is another promising avenue for future research.

\section{Conclusion}\label{sec: Conclusion}
In this work, we introduced the multi-view graph-tuple framework to address a fundamental challenge of applying GNNs to data with continuous relationships. Standard single-graph approaches face a difficult trade-off: either constructing a weak-connection graph via thresholding, which inevitably discards contextual information, or using the full-graph (complete) graph, which often incurs higher computational costs. Our framework resolves this limitation by explicitly partitioning the interaction space into a graph tuple, comprising a strong-connection graph for primary interactions and a weak-connection graph for global context, and performing (heterogenous) message-passing in parallel to maintain efficiency. We theoretically show the expressivity improvements of our multi-view graph-tuple model over the single-graph models. We also empirically validate our framework through experiments on molecular property prediction and cosmological parameter prediction, showing that our multi-view approach can achieve an overall better performance against single-graph baselines.

As a proof of concept, we create multiple graph views using a fixed partitioning strategy depending on the edge feature values, which may be sub-optimal, as observed in the cosmological parameter prediction experiments. 
Future work could explore adaptive mechanisms, such as a learnable threshold or flexible relationships between scales, to allow the multi-view graph-tuple framework to tailor its structure to the specific task and data. Another interesting direction is to apply our framework for other dense-graph applications, such as brain connectomes and combinatorial optimization problems.

\bibliographystyle{unsrtnat}
\bibliography{reference}

\if0\blind
{
 \clearpage
  \onecolumn

\newpage
\appendix

\section{Full Details for Theoretical Analysis}
\label{app:theory_details}

This section provides the complete definitions, lemmas, and proofs for the theoretical results presented in Section~\ref{sec:theory_main}.

\subsection{Setup and Model Classes}

Let $S_1, S_2 \in \mathbb{R}^{n\times n}$ be two graph shift operators on a common node set, representing two different connection strenghts (e.g., strong and weak). We also define a dense-graph operator $S_{\mathrm{dense}} = S_1 + S_2$ which includes both the strong and weak connections. For a fixed polynomial degree bound $m \in \mathbb{N}$, we consider three classes of linear graph filters.

\begin{definition}[Filter Classes]\label{def:classes}
The strong-connection ($S_1$-based), dense-graph, and multi-view graph tuple filter classes are defined as:
\[
\begin{aligned}
    H_1(m)      &= \Big\{\, p(S_1) = \sum_{k=0}^m a_k S_1^k \,\Big\}, \\
    H_0(m)      &= \Big\{\, q(S_{\mathrm{dense}}) = \sum_{k=0}^m b_k (S_1+S_2)^k \,\Big\}, \\
    H_{\mathrm{Gt}}(m) &= \Big\{\, g(S_1,S_2) = \sum_{w \in \mathcal{W}_{\le m}} c_w \, w(S_1,S_2) \,\Big\},
\end{aligned}
\]
where $\mathcal{W}_{\le m}$ is the set of all words of length up to $m$ formed from the symbols $\{1,2\}$, and $w(S_1, S_2) = \prod_{t=1}^{|w|} S_{w_t}$.
\end{definition}

\begin{remark}[Degrees of Freedom]
The classes $H_1(m)$ and $H_0(m)$ are defined by $m+1$ free parameters each. In contrast, the multi-view graph-tuple class $H_{\mathrm{Gt}}(m)$ is defined by $2^{m+1}-1$ parameters, reflecting its ability to assign an independent coefficient to every possible path of length up to $m$.
\end{remark}

\subsection{Expressivity and Risk Dominance}

Our analysis is grounded in a standard linear data model and the corresponding prediction risk.

\begin{assumption}[Data model and risk]\label{as:data}
The input $x\in\mathbb R^n$ is zero-mean with covariance $\Sigma\succ0$. For a linear predictor $M\in\mathbb R^{n\times n}$ and target $y^\star$, we define the risk as
\[
R(M)=\mathbb E\big[\|Mx-y^\star\|_2^2\big].
\]
We assume $y^\star=M^\star x+\varepsilon$ with $\mathbb E[\varepsilon\mid x]=0$ for some oracle predictor $M^\star \in \overline{H_{\mathrm{Gt}}(m)}$, where $ \overline{H_{\mathrm{Gt}}(m)}$ is the closure of $H_{\mathrm{Gt}}(m)$.
\end{assumption}

For matrices $A,B\in\mathbb R^{n\times n}$, we define the weighted inner product, Frobenius norm, and the corresponding distance:
\[
\langle A,B\rangle_\Sigma=\mathrm{tr}(B^\top A\,\Sigma),\qquad
\|A\|_{\Sigma,F}^2=\langle A,A\rangle_\Sigma, \qquad
\mathrm{dist}_\Sigma(A, B) = \|A - B\|_{\Sigma,F}.
\]
We use $\mathrm{dist}_\Sigma(M, \mathcal{C}) = \inf_{M' \in \mathcal{C}} \mathrm{dist}_\Sigma(M, M')$ to denote the distance from a matrix $M$ to a set $\mathcal{C}$.

Our analysis relies on two lemmas. The first recasts the prediction risk as a best approximation problem in a matrix space, and the second provides a combinatorial expansion.

\begin{lemma}[Risk Decomposition]
\label{lem:risk_decomp}
Under Assumption~\ref{as:data}, for any predictor $M$, the risk can be decomposed as $R(M) = R(M^\star) + \|M - M^\star\|_{\Sigma,F}^2$.
\end{lemma}
\begin{proof}
By definition, $R(M) = \mathbb{E}[\|(M-M^\star)x - \varepsilon\|_2^2]$. Expanding this and taking the expectation, the cross-term $\mathbb{E}[\langle (M-M^\star)x, \varepsilon \rangle]$ vanishes due to the condition $\mathbb{E}[\varepsilon|x]=0$. The remaining terms are $\mathbb{E}[\|(M-M^\star)x\|_2^2] = \|M-M^\star\|_{\Sigma,F}^2$ and $\mathbb{E}[\|\varepsilon\|_2^2] = R(M^\star)$, which yields the result.
\end{proof}

\begin{lemma}[Noncommutative Binomial Expansion]\label{lem:nc_binomial}
For any integer $k \ge 0$, we have $(S_1+S_2)^k = \sum_{w\in\mathcal{W}_k} w(S_1,S_2)$, where $\mathcal{W}_k$ is the set of words of length $k$.
\end{lemma}

These lemmas allow us to establish our main theoretical results concerning the expressivity and risk of the multi-view graph-tuple class.

\begin{proposition}[Expressivity]\label{thm:expressivity}
For any $m \ge 0$, the multi-view graph-tuple class contains the strong-connection and dense-graph classes: $H_1(m) \subseteq H_{\mathrm{Gt}}(m)$ and $H_0(m) \subseteq H_{\mathrm{Gt}}(m)$. If the operators do not commute, i.e., $[S_1, S_2] = S1S2-S2S1 \neq  0$, and $m \ge 2$, this latter inclusion is strict: $H_0(m) \subsetneq H_{\mathrm{Gt}}(m)$.
\end{proposition}

\begin{proof}
The inclusion $H_1(m) \subseteq H_{\mathrm{Gt}}(m)$ is trivial by construction. The inclusion $H_0(m) \subseteq H_{\mathrm{Gt}}(m)$ follows from Lemma~\ref{lem:nc_binomial}, which shows that any polynomial in $S_{\mathrm{dense}}$ is a sum over words with coefficients tied according to their length ($c_w = b_{|w|}$). The inclusion is strict under non-commutativity because an element like the commutator $[S_1, S_2]$ is in $H_{\mathrm{Gt}}(m)$ but not in $H_0(m)$, as the latter requires the coefficients of $S_1S_2$ and $S_2S_1$ to be equal.
\end{proof}

\begin{remark}[The Commuting Case]
Even if $[S_1,S_2]=0$, $H_0(m)$ generally remains a proper subset of $H_{\mathrm{Gt}}(m)$ unless $S_1$ and $S_2$ are algebraically dependent (e.g., $S_2 = cS_1$), because $H_0(m)$ still enforces coefficient tying across all same-degree terms.
\end{remark}

The greater expressivity of the multi-view graph-tuple class can translate into improved generalization performance.

\begin{proposition}[Oracle risk dominance]\label{thm:quantitative_gap}
Let $\mathcal{U}(m) = \overline{H_{\mathrm{Gt}}(m)}$ and $\mathcal{V}(m) = \overline{H_0(m)}$ be the closures of the multi-view graph-tuple and dense-graph classes. If $m \ge 2$ and the oracle predictor $M^\star$ has a non-zero component in the orthogonal complement of $\mathcal{V}(m)$ within $\mathcal{U}(m)$ (i.e., $\Pi_{\mathcal{V}(m)^\perp}(M^\star) \neq 0$), then the multi-view graph-tuple class achieves a strictly lower oracle risk. The performance gap is given precisely by:
\[
\inf_{q \in H_0(m)} R(q) - \inf_{g \in H_{\mathrm{Gt}}(m)} R(g) = \|\Pi_{\mathcal{V}(m)^\perp}(M^\star)\|_{\Sigma,F}^2 > 0.
\]
\end{proposition}
\begin{proof}
By Lemma~\ref{lem:risk_decomp}, the minimum risk for a closed class $\mathcal{C}$ is $\inf_{M \in \mathcal{C}} R(M) = R(M^\star) + \mathrm{dist}_{\Sigma}(M^\star, \mathcal{C})^2$. Subtracting the expressions for $\mathcal{C} = \mathcal{U}(m)$ and $\mathcal{C} = \mathcal{V}(m)$ yields the risk gap:
\[
\inf_{q \in H_0(m)} R(q) - \inf_{g \in H_{\mathrm{Gt}}(m)} R(g) = \mathrm{dist}_{\Sigma}(M^\star, \mathcal{V}(m))^2 - \mathrm{dist}_{\Sigma}(M^\star, \mathcal{U}(m))^2.
\]
Since $M^\star \in \mathcal{U}(m)$ by assumption, $\mathrm{dist}_{\Sigma}(M^\star, \mathcal{U}(m))$ is zero. Since the shortest distance from $M^\star$ to the subspace $\mathcal{V}(m)$ is the norm of its component in the orthogonal complement, $\mathrm{dist}_{\Sigma}(M^\star, \mathcal{V}(m))^2 = \|\Pi_{\mathcal{V}(m)^\perp}(M^\star)\|_{\Sigma,F}^2$. Substituting this gives the claimed result. As the proposition's premise is that this projection is non-zero, the squared norm is strictly positive.
\end{proof}

\begin{corollary}[Sufficient Condition for Strict Improvement]
\label{cor:sufficient}
The condition for strict risk dominance in Proposition~\ref{thm:quantitative_gap} is satisfied if $m \ge 2$, the operators do not commute, $[S_1, S_2] \neq 0$, and the degree-2 component of $M^\star$ contains a non-zero multiple of the commutator $[S_1, S_2]$.
\end{corollary}
\begin{proof}[Proof]
This follows because, as established in the proof of Proposition~\ref{thm:expressivity}, the commutator $[S_1, S_2]$ is an element of the orthogonal complement $\mathcal{V}(m)^\perp$. If $M^\star$ contains a non-zero multiple of this element, its projection onto this subspace, $\Pi_{\mathcal{V}(m)^\perp}(M^\star)$, must be non-zero.
\end{proof}

In summary, the ability of our multi-view graph-tuple framework to assign distinct weights to distinct interaction paths makes it more expressive than models constrained to polynomials of a single operator. This greater expressivity guarantees a lower or equal modeling risk for target functions that satisfies our modeling assumptions: the oracle predictor $M^\star$ that is expressible within our framework (i.e., $M^\star \in \overline{H_{\mathrm{Gt}}(m)}$ in Assumption~\ref{as:data}).

These theoretical results directly apply to the linear backbone of the GNNs used in our experiments. Specifically, our analysis focuses on these linear operators and does not consider the nonlinear activation functions applied to their outputs. While a full characterization of the nonlinearities is more complex, our analysis provides a clean conceptual baseline: any nonlinear architecture built upon this backbone inherits the fundamental expressivity gap between the underlying operator classes. Extending such guarantees to fully nonlinear settings is an interesting but technically nontrivial direction that we leave for future work.

\section{Implementation Details}
\label{app:implementation_details}

\subsection{Molecular Property Prediction}
\label{app:Molecular Property Prediction}

\subsubsection{Feature Construction.}
Since the QM7b dataset is feature-scarce, we first construct node and edge features from the molecule's Coulomb matrix, $X$. Following prior work~\cite{blum2024learning}, we derive initial features directly from the Coulomb matrix entries. Specifically, we apply a "binary expansion" technique to expand the scalar diagonal entries ($X_{ii}$) and off-diagonal entries ($X_{ij}$) into 100-dimensional vectors, which serve as the initial node and edge features, respectively. These raw scalar values are then projected into a 100-dimensional hidden space by learnable encoders.

\begin{table}[htbp]
\centering
\caption{Percentage of remaining strong edges under different Coulomb thresholds ($c$).}
\label{tab:stats}
\renewcommand{\arraystretch}{1.2}
\begin{tabular}{l ccccc}
\toprule
\textbf{Threshold ($c$)} & 0 & 0.5 & 1 & 2 & 2.5 \\
\midrule
\textbf{Remaining Edges (\%)} & 100.00 & 68.02 & 50.10 & 25.89 & 24.90 \\
\bottomrule
\end{tabular}
\end{table}

\subsubsection{Model Architecture and Training.}
All GNN models are constructed with two GNN layers and a hidden dimension of 100. The MLPs within each GINEConv layer consist of two linear layers separated by a ReLU activation. The edge encoders within each path of the GINE-Gt model are implemented as single linear layers. For graph-level prediction, we apply a global mean pooling to the node features of the final GNN layer, and the resulting graph vector is passed through a 3-layer MLP with a ReLU activation to produce the final output.

For training, all models use a batch size of 128. We train the models by minimizing the L1 Loss (Mean Absolute Error) using the Adam optimizer~\cite{kingma2014adam} with an initial learning rate of $5\times10^{-3}$ and weight decay of $10^{-5}$. A cosine–plateau scheduler reduces the learning rate by a factor of 0.8 after five epochs without validation improvement (minimum $10^{-5}$). Early stopping is triggered after 20 idle epochs or when the run reaches a maximum of 1000 epochs.

\subsubsection{Evaluation Protocol and Environment.}
To ensure a robust evaluation, we employ a stratified ten-fold cross-validation scheme. For each fold, we reserve 10\% of the data for testing, while the remainder is split into a 9:1 train/validation ratio. Then we report the mean and standard error of the Mean Absolute Error (MAE) across the ten test set folds. These experiments were performed on a MacBook Air (15-inch, 2023) featuring an Apple M2 processor and 16 GB of unified memory, running macOS Ventura (13.4). All models were implemented in PyTorch.

\subsection{Cosmological Parameter Inference}
\label{app:Cosmological Parameter Inference}

\subsubsection{Feature Construction.}
For all constructed graphs, since the dark matter halos are treated as identical particles, the initial feature for each node is set to a 1-dimensional unit vector ($h_i^{(0)}=[1]$). This vector is then projected into the model's hidden dimension by an embedding layer. Edge attributes are dynamically generated by expanding the Euclidean distance between halos into a 32-dimensional feature vector using a Radial Basis Function (RBF) encoding.

\begin{table}[h!]
\centering
\caption{Search space for the strong-connection cutoff radius ($c_1$) for different datasets and target parameters. The search for $c_1$ is performed with a step of 0.5. The weak-connection radius ($c_2$) is always set to $2c_1$, so its corresponding search is performed with a step of 1.0.}
\label{tab:radii_search_space}
\begin{tabular}{@{}lccc@{}}
\toprule
\textbf{Dataset} & \textbf{Target Parameter} & \textbf{$\mathbf{c}_\mathbf{1}$ Values (Mpc/h)} & \textbf{$\mathbf{c}_\mathbf{2}$ Values (Mpc/h)} \\
\midrule
CAMELS (TNG)     & $\Omega_m$                & 0.5, 1.0, \dots, 3.5          & 1.0, 2.0, \dots, 7.0          \\
CAMELS (TNG)     & $\sigma_8$                & 1.5, 2.0, \dots, 5.0          & 3.0, 4.0, \dots, 10.0         \\
CAMELS-SAM       & $\Omega_m$                & 0.5, 1.0, \dots, 3.5          & 1.0, 2.0, \dots, 7.0          \\
CAMELS-SAM       & $\sigma_8$                & 1.0, 1.5, \dots, 4.0          & 2.0, 3.0, \dots, 8.0          \\
\bottomrule
\end{tabular}
\end{table}

\subsubsection{Model Architecture and Training.}
Our EGNN-Gt models are constructed with 3 layers and a hidden dimension of 96. The MLPs within each EGCL operator use the SiLU activation function. For the primary task of cosmological parameter prediction, a global mean pooling is applied to the final node features, and the resulting graph-level representation is passed through a 2-layer MLP to produce the output. We train all models for a maximum of 300 epochs by minimizing the Mean Squared Error (MSE) loss, using a batch size of 8. The AdamW optimizer is used with an initial learning rate of $5 \times 10^{-4}$ and a weight decay of $1 \times 10^{-5}$. The learning rate is dynamically adjusted using a ReduceLROnPlateau scheduler, which reduces it by a factor of 0.7 if the validation loss does not improve for 5 consecutive epochs, down to a minimum of $1 \times 10^{-5}$.

\subsubsection{Evaluation Protocol and Environment.}
The datasets are randomly partitioned into training (60\%), validation (20\%), and test (20\%) sets. To ensure the robustness of our findings, each experiment is repeated 10 times with different random seeds, and we report the mean and standard error of the performance metrics on the test set. All experiments were conducted on a single NVIDIA RTX 6000 Ada Generation GPU, equipped with 48 GB of VRAM.

}


\end{document}